%% file: EremeevSilaevTopchii_arxiv.tex
\documentclass{article} 
\usepackage{mathai2026_conference,times}
\usepackage{graphicx}

\input{mathai2026_math_commands.tex}

\usepackage{hyperref}
\usepackage{url}
\input{preamble}
\newtheorem{lemma}{Lemma}
\newtheorem{theorem}{Theorem}
\newtheorem{definition}{Definition}
\newtheorem{corollary}{Corollary}

\title{Generalized Heavy-tailed Mutation for Evolutionary Algorithms}

\author{Anton V. Eremeev, Dmitri V. Silaev \& Valentin A. Topchii 
\\
$\mbox{ }$
\\
Novosibirsk State University\\
1, Pirogova str., Novosibirsk, 630090, Russia,\\
Sobolev Institute of Mathematics\\
4, Koptyuga pr., Novosibirsk, 630090, Russia\\
\texttt{eremeev@ofim.oscsbras.ru}}

\mathaifinalcopy 

\begin{document}

\maketitle

\begin{abstract}
The heavy-tailed mutation operator, proposed by Doerr, Le, Makhmara, and Nguyen (2017) for evolutionary algorithms, is based on the power-law assumption of mutation rate distribution. Here we  generalize the power-law assumption using a regularly varying constraint on the distribution function of mutation rate. In this setting, we generalize the upper bounds on the expected optimization time of the $(1+(\lambda,\lambda))$ genetic algorithm obtained by Antipov, Buzdalov and Doerr (2022) for the OneMax function class parametrized by the problem dimension $n$. In particular, it is shown that, on this function class, the sufficient conditions of Antipov, Buzdalov and Doerr (2022) on the heavy-tailed mutation, ensuring the $O(n)$ optimization time in expectation, may be generalized as well. This optimization time is known to be asymptotically smaller than what can be achieved by the $(1+(\lambda,\lambda))$ genetic algorithm with any static mutation rate.
 A new version of the heavy-tailed mutation operator is proposed, satisfying the generalized conditions, and promising results of computational experiments are presented.
\end{abstract}


\section{Introduction}

A distinctive feature of evolutionary algorithms (EA) is their imitation of the process of evolutionary adaptation of a biological population to environmental conditions. Individuals correspond to trial points in the solution space of an optimization problem, and their fitness is determined by the values of the objective function, taking into account penalties for violating the problem's constraints, if any. The construction of new trial points in EA is accomplished using mutation and crossover operators. When using the latter, EAs are commonly referred to as genetic algorithms. When solving unconstrained pseudo-Boolean
maximization problems $\max\{f(x) : x\in \{0,1\}^n\}$, or minimization problems
$\min\{f(x) : x\in \{0,1\}^n\}$, where the objective function is ${f:\{0,1\}^n
\to \mathbb R}$, one of the most frequently used mutation operators
is the {\it standard mutation} \cite{bib:Gold89}, where
each bit of the given string $x\in \{0,1\}^n$ independently
changes its value with a given probability $p$. In this paper,
we will assume that in the case of standard mutation, at each iteration of an EA with distribution $\Bin(n, p)$, a number of mutated
bits~$\ell$ is selected, and the next descendant is obtained from the parent
solution by making changes to  $\ell$ randomly selected bits.

The main performance characteristic of evolutionary algorithms
in solving optimization problems is the {\em optimization time}, hereafter denoted by $T$, which is defined as the number of times the fitness function is evaluated until the optimum is reached for the first time. Typically, the expected value of the optimization time or the average optimization time is studied.
In this paper, we study the optimization time of a genetic algorithm with a computational scheme from~\cite{ABD22} when maximizing the fitness function $\onemax(x)= \sum_{i=1}^n x_i$. A family of such functions, parameterized by the problem dimension~$n$, is one of the basic benchmarks in the theory of evolutionary computation, used to evaluate the effectiveness of EAs on simple problems. In particular, in the case of high expected optimization time on \onemax, random search algorithms or subsets of their adjustable hyper-parameter values are considered ineffective~\cite{bib:Lehre2011,bib:oliveto15}.


Symmetry considerations imply that EAs based
on the standard mutation behave identically on both \onemax and any function ${\onemax_z: \{0, 1\}^n \to \mathbb R}$,
defined as
$$
\onemax_z(x):=|\{i \in \{1,...,n\}: x_i = z_i\}|
$$
for any bit string $z\in \{0, 1\}^n$.
It was shown by~\cite{Erdos} that any
algorithm that obtains information only from queries of the fitness function $\onemax_z$ will, on average, require $
\Omega(n/\log n)$ fitness function computations before first
hitting the optimal solution $z$, and this bound cannot
be improved.

\cite{DDE15} developed a genetic algorithm, $(1 +
(\lambda, \lambda))$~GA, with a new crossover operator that eliminates ``bad'' mutations. At each iteration of the algorithm, $(1 + (\lambda, \lambda))$~GA, a single parent individual
generates $\lambda=\lambda(n)$ offspring independently of each other at
equal Hamming distance~$\ell$ from the parent, with $\ell \sim
{\rm Bin}(n,p)$. Next, the best-fitting of these solutions is selected, and a crossover operator with parameter $c$ is applied to it. With probability $c$, the crossover operator
uses bits from the best offspring, and with probability $1-c$, it uses bits
from the parent solution. In this way, $\lambda$
individuals are created, and the best of these $\lambda$ individuals is accepted as a
new parent if it is at least as fit as the previous parent. A theoretical analysis of the optimization time
has shown that the $(1 + (\lambda, \lambda))$~GA algorithm, for many
values of the adjustable parameters, is asymptotically
faster by \onemax than most classical evolutionary
algorithms.

As demonstrated by~\cite{DLMN17}, choosing a mutation parameter for multi-extremal problems is significantly more challenging than for \onemax. To overcome this difficulty, \cite{DLMN17} proposed using a random choice for the mutation parameter $p$
following a heavy-tailed distribution, namely, a truncated power-law distribution with exponent $-\beta <- 1$.
In this case, a random number $\alpha
\in \{1,\ldots, \lfloor \frac{n}{2}\rfloor\}$ is first selected, so that the probability
of selecting $\alpha=k$ is proportional to $k^{-\beta}$, $k\le
\lfloor \frac{n}{2}\rfloor$. Then, $p=\alpha/n$ is set and
the standard mutation is applied with this value~$p$.
At each iteration of the EA, the values of $\alpha$ are selected independently.

As shown by~\cite{DD18} for the \onemax case, the optimal choice of a fixed value of the mutation parameter p for the entire running time of the GA yields an optimization time of 
\begin{equation} \label{eqn:optimax_fixed} E[T]=\Theta\left(n \sqrt{\frac{\log(n)\log\log\log(n)}{\log\log n}}\right),
\end{equation}
which is asymptotically smaller than the average optimization time of many known evolutionary algorithms.

\cite{ABD22} demonstrated the effectiveness of fast mutation in optimizing the \onemax function. Here, the $(1+(\lambda, \lambda))$~GA from~\cite{DDE15}, combined with the fast mutation operator, was considered. \cite{ABD22} proved an upper bound for the average optimization time of order $O(n)$  for $(1+(\lambda, \lambda))$~GA with fast mutation on \onemax, under a specific choice of the distributions of the random variables
$\lambda$ and $p$. This is less than the average optimization time of
$(1+(\lambda, \lambda))$~GA with any fixed probability
of mutations.
In this algorithm, both the population size~$\lambda$ and the
rapid mutation parameter~$p$ have a truncated power-law distribution
with upper bounds $\lambda\le u_n$ and $p\le u_n/n$,
respectively. The linear bound of~\cite{ABD22} holds when
the power-law exponent $\beta$ satisfies the inequalities $2 <
\beta< 3$ and $u_n \ge \ln^{1/(3-\beta)} n$.

{\bf The main result} of this work, presented in 
Theorems~\ref{t1} and \ref{t2}, shows that the upper bounds on 
the expected optimization time of $(1 + (\lambda, \lambda))$~GA, similar to 
those obtained in~\citep{ABD22}, hold not only for truncated 
power-law distributions of the random variables $\lambda$ and $p$, but also 
for a wider class of distributions described in terms of 
regularly varying constraints on the distribution function of this quantity. As follows from 
(\ref{eqn:optimax_fixed}), the linear bound on the average 
optimization time obtained by us, like the linear bound from~\citep{ABD22}, 
turns out to be asymptotically smaller than the average optimization time of $(1 + 
(\lambda, \lambda))$~GA with any mutation parameter~$p$ that remains unchanged during the algorithm's execution.

The proofs are provided in the Appendix.

This is a translation of the publication in Russian~\citep{EST24}. A special case of this result, obtained without using
the apparatus of regularly varying functions, is presented in~\citep{GECCO24}.

\subsection{Notation and definitions}

Denote $\mathbb{N}_{m}:=\{k:k\in\mathbb{N},k\leq m\}$,
$S:=\{0,1\}^{n}$, $|S|=2^{n}$.
Introduce the norm and Hamming distance $|x|=\sum_{i=1}^{n} x_{i}$ and 
$|x-y|=\sum_{i=1}^{n}|x_{i}-y_{i}|$ for $x,y \in S$. Denote 
by $x^*$ the solution with ones in all positions, and $Z_{s}=\{x\in S : |x-x^*|=s\}$ is the set of solutions having exactly $s$ zeros, 
$s=0,\dots,n$. In particular, $Z_0=\{x^*\}$.

Let $\lambda(n)$, $n\in\mathbb{N}$ be a set of jointly independent 
random variables (r.v.s) with ranges from a subset of $\mathbb{N}_{u_{n}}$, where $u_{n}\leq 
0.5n$ is the maximum possible value of $\lambda(n)$ with 
positive probability. In other words, probabilities 
$\mathbf{P}(\lambda(n)=k)=p_{n,k}\geq 0$ only for 
$k\in\mathbb{N}_{u_{n}}$ and $p_{n,u_{n}}>0$. Constraints on 
$p_{n,k}$ will be introduced later. The r.v.s $\lambda_{*}(n)$ with different indices 
instead of $*$ are independent and identically distributed with $\lambda(n)$. We 
study the case $u=u(n)\to\infty$ as $n\to\infty$.

Let us describe the algorithm $\mathcal{A}$ under study, whose 
predecessor is the $(1+(\lambda,\lambda))$~EA from \citep{DDE15} with 
deterministic mutation probabilities $p=\lambda(n)/n$ and 
crossover parameter $c=\lambda^{-1}(n)$. Later 
in~\citep{ABD22}, a randomization of the $(1+(\lambda,\lambda))$~EA from \citep{DDE15} was proposed with respect to the population 
size $\lambda(n)$, where the probabilities $\mathbf{P}
(\lambda(n)=k)=p_{n,k}$ are expressed via power functions. 
This algorithm in~\citep{ABD22} was named the {\em fast $(1+
(\lambda,\lambda))$ genetic algorithm}. 
When the parameter $\lambda(n)$ is chosen according to a power law and the mutation 
parameter $p=\lambda(n)/n$, the algorithm from~\citep{ABD22} has a 
heavy-tailed mutation, coinciding with that proposed in~\citep{DLMN17}. 
In this paper, we abandon explicit expressions for $p_{n,k}$ and 
provide only constraints on the distribution function of the r.v. 
$\lambda(n)$. In all other respects, algorithm $\mathcal{A}$ 
coincides with the fast $(1+(\lambda,\lambda))$ genetic 
algorithm.

{\em Algorithm $\mathcal{A}$: $(1+(\lambda,\lambda))$ genetic algorithm with regularly varying constraints on the 
distribution function of $\lambda(n)$ with an upper limit of its values
$u_{n},$ maximizing the function $f:\{0,1\}^{n}\to \mathbb{R}$} 

\medskip 
\ \,1. $x$ $\leftarrow$ random bit string of length $n$;

\ \,2. \textbf{while} \textit{not terminated} \textbf{do}

\ \,3.  \quad Choose $\lambda$ from [1..u] with $\mathbf{P}(\lambda=k)=p_{n,k}$;

\ \,4.  \quad Choose $\ell \sim {\rm Bin}(n,\lambda(n)/n)$;

\ \,5.  \quad \textbf{for} $i\in[1..\lambda]$ \textbf{do}

\ \,6.  \quad   \quad  $x^{(i)}$ $\leftarrow$ a copy of $x$;

\ \,7.  \quad   \quad   Flip $\ell$ bits in $x^{(i)}$ chosen uniformly at random;

\ \,8.  \quad \textbf{end}

\ \,9.  \quad $x' \leftarrow {\rm arg\, max}_{z\in
\{x^{(1)},...,x^{(\lambda)}\}} f(z)$;

10.  \quad \textbf{for} $i\in[1..\lambda]$ \textbf{do}

11.  \quad   \quad  Create $y^{(i)}$ by taking each bit from $x'$ with probability  

\ \ \, \;  \quad   \quad $\lambda^{-1}$ and from $x$ with probability $(\lambda-1)\lambda^{-1}$;

12.  \quad \textbf{end}

13.  \quad  $y \leftarrow {\rm arg\, max}_{z\in
\{y^{(1)},...,y^{(\lambda)}\}} f(z)$;

14.  \quad  \textbf{if} $f(y) \geq f(x)$ \textbf{then}

15.   \quad   \quad $x \leftarrow y$;

16.   \quad \textbf{end}

17. \textbf{end}
\medskip 

Algorithm~$\mathcal{A}$ starts with a random initial 
bit string $x$. Each iteration of the algorithm consists of obtaining 
a random realization of the r.v. $\lambda(n)$ for the population 
size~$\lambda$, followed by a mutation phase, a 
crossover phase, and a selection phase. In the mutation phase (lines 4--8), after 
obtaining a realization $\ell$ with distribution $\Bin(n, p)$, 
$\lambda$ offspring are created from $x$ by making changes in $\ell$ 
randomly selected bits in each of them. From these new $\lambda$ 
individuals, one with the highest fitness is selected for further 
participation in the crossover phase. If there is more than one offspring with 
maximum fitness, we choose one of them uniformly at random. 
Denote the selected offspring by $x'$.
In the crossover phase (lines 10--12), $\lambda$ new 
offspring are created from the parent individual $x$ and the winner $x'$ from the mutation phase. 
From these $\lambda$ offspring, the string with the highest 
fitness is selected. If there are several, then we choose 
uniformly at random one of them (strings that coincide with the parent 
solution $x$ are not considered). In the selection phase 
(line 14), the parent $x$ is replaced by the winner individual of the crossover phase $y$, if the fitness of $y$ is not lower than 
the fitness of $x$. In algorithm~$\mathcal{A}$, as in many 
other algorithms mentioned above, no stopping criterion is specified. 
This is because in theoretical studies we are mainly 
interested in the time of first reaching the optimum. In 
practical applications of algorithm~$\mathcal{A}$, it is natural to 
specify a stopping criterion.

\section{Main properties of the algorithm}

Denote by $\ell_{\lambda(n)}(s)$ the number of iterations of Algorithm~1 
from \citep{DDE15}, starting from a solution $x\in Z_{s},$ until the first 
improvement of the fitness function for fixed $n$ and 
$\lambda(n)$, and the probability of improvement on the first iteration 
by $p_{\lambda(n)}(s)=\mathbf{P}(\ell_{\lambda(n)}(s)=1)$. Further 
we study algorithm $\mathcal{A}$ with the fitness function \onemax, and the probabilities $p_{\lambda(n)}(s)$ coincide for 
all $x\in Z_{s}$. Note that with a random choice of $\lambda(n)$, 
the r.v. $\ell_{\lambda(n)}(s)$ does not carry a semantic load, but only 
the event $\{\ell_{\lambda(n)}(s)=1\}$ is important.

We present Lemma 7 from \citep{DDE15} with a fixed (not random) 
value of $\lambda$ for each $n$ with mutation
parameter $p=\lambda/n$ and crossover parameter $c=\lambda^{-1}$.

\begin{lemma} \label{l1} One iteration of algorithm~$\mathcal{A}$ with 
fixed $\lambda$ in the case of the fitness function 
\onemax, starting from a solution $x\in Z_{s}$, leads to an improvement 
of the current fitness function value with probability 
satisfying the inequality
\begin{equation} \label{C11} 
p_{\lambda}(s)\geq C\left(1-\left(1-s/n \right)^{\lambda^{2}/2} \right)\left(1-e^{-1/8}\right)
\end{equation}
for some constant $C>0$ independent of $n$.
\end{lemma}

Further, without additional comments, for constants independent of $n$ 
we will use the notations $C>0$ and $c>0$ with or without 
indices. These constants may be interrelated, but 
only the presence of these positive constants is fundamentally important. In some cases, to simplify 
the presentation, for different constants we will use the symbols 
$C^{*}$ and $c^{*}$ without indicating their explicit form, and even in one 
equation, in different parts of the equality, these values may be 
different.

Inequality (\ref{C11}) taking into account the bound
$$
1-(1-p)^{\lambda}\geq \dfrac{\lambda p}{1+\lambda p},\ \forall
p\in(0,1), \ \lambda>0,
$$
from Lemma 2 of \citep{ABD22} can be written as
\begin{equation} \label{C12} 
p_{\lambda(n)}(s)\geq C_{1}\dfrac{0.5\lambda^{2}(n)s/n}{1+0.5
\lambda^{2}(n)s/n},
\end{equation}
in particular, inequality (\ref{C12}) can be written as
\begin{eqnarray} \label{C13}
p_{\lambda(n)}(s)&\geq& C_{2} \lambda^{2}(n)s/n, \mbox{ for } 
\lambda^{2}(n)s/n<1,\\ \label{C14}
p_{\lambda(n)}(s) &\geq& C_{3}, \mbox{ for } \lambda^{2}(n)s/n\geq 1.
\end{eqnarray}

Let $\lambda(n)$ be random. Formally, at each iteration with 
number $t$, a r.v. $\lambda_{t}(n)$ is realized with the same 
distribution as $\lambda(n)$. Given that there is no dependence on 
$t$, we omit this index. Using the law of total 
probability (averaging over $\lambda(n)$), denote the 
probability of improvement in one iteration by $p_{n}(s)=
\mathbf{E}_{\lambda(n)}p_{\lambda(n)}(s)=\mathbf{P}(\ell_{n}(s)=1)
$, where $\ell_{n}(s)$ is the number of iterations until improvement of the fitness
function. Formally, in the definition of the probability $p_{n}(s)$, 
which does not depend on the iteration number $t$, the event $\{\ell_{n}(s)=1\}$ 
depends on this number $t$, and $\ell_{n}(s)$ equals the number of 
unsuccessful iterations in a series of iterations until the first successful one.
The count of iterations starts from the initial individual
or after the next increase in the fitness function.

Compute the expectation $\mathbf{E}\ell_{n}(s)$. The random variable 
$\ell_{n}(s)$ has a geometric distribution with parameter 
$p_{n}(s)$. By definition we have 
\begin{eqnarray} \nonumber
\mathbf{P}(\ell_{n}(s)=k)&=&(1-p_{n}(s))^{k-1}p_{n}(s),\ 
k\in \mathbb{N},\\ \label{C15}
\mathbf{E}\ell_{n}(s)&=&(1-p_{n}(s))p_{n}^{-1}(s)+1
=p_{n}^{-1}(s).
\end{eqnarray}

Suppose that for some fixed $m_{0},m_{1}\in
\mathbb{N}$ there exists a constant $c>0$ independent of $n$ such that 
for all $n\in\mathbb{N}\backslash\mathbb{N}_{m_{1}}$ the 
inequalities hold 
\begin{equation} \label{Ca23} 
\mathbf{P}(\lambda(n)\in \mathbb{N}_{m_{0}})\geq c.
\end{equation} 
A sufficient condition for inequality (\ref{Ca23}) to hold 
is uniform in $n$ boundedness of $\mathbf{E}\lambda(n)$,
i.e., the conditions $\mathbf{E}\lambda(n)
\leq c_{1}$, $\forall n\in\mathbb{N}$, for some $0<c_{1}<\infty$. Then by Markov's inequality for non-negative random variables $\lambda(n)$ for any $c_{0}>c_{1}$
\begin{equation*}
\mathbf{P}(\lambda(n)<c_{0})\geq 1 - \mathbf{E}\lambda(n)/c_{0}
\geq 1-c_{1}/c_{0}>0.
\end{equation*}

Conditions for inequality (\ref{Ca23}) to hold will be 
formulated as: for sufficiently large $n$. Obviously,
if inequality (\ref{Ca23}) holds for some fixed 
$m_{0}$ and $m_{1}$, then it also holds when these values are replaced by 
any fixed values greater than them.

Estimate the fraction on the right-hand side of bound (\ref{C12}) for $\lambda(n)
\in \mathbb{N}_{m_{0}}$. If ${s\leq 0.5nm_{0}^{-2}}$, then
$0.5\lambda^{2}(n)s/n\leq0.25$ and the bound holds
$$
\dfrac{0.5\lambda^{2}(n)s/n}{1+0.5\lambda^{2}(n)s/n}\geq 
0.4\lambda^{2}(n)s/n\geq 0.4 sn^{-1}.
$$
Otherwise, $s> 0.5nm_{0}^{-2}$ and $(0.5\lambda^{2}(n)
s/n)^{-1}< 4$ and the bound holds
$$
\dfrac{0.5\lambda^{2}(n)s/n}{1+0.5\lambda^{2}(n)s/n}>0.2= 
0.2sn^{-1}s^{-1}n\geq 0.2sn^{-1}.
$$

When inequality (\ref{Ca23}) holds, due to the last two bounds 
and inequality (\ref{C12}), for all $n\in \mathbb{N}
\backslash\mathbb{N}_{m_{1}}$ and $s\in \mathbb{N}_{n}$
the inequalities hold
\begin{equation} \label{C16} 
p_{n}(s)\geq \mathbf{E}_{\lambda(n)}\{p_{\lambda(n)}(s);\lambda(n)
\in \mathbb{N}_{m_{0}}\}\geq \dfrac{0.2C_{1}s}{n}\mathbf{P}
(\lambda(n)\in \mathbb{N}_{m_{0}})=\dfrac{C_{4}s}{n}.
\end{equation}

\begin{lemma}  \label{l2}
In the case of the fitness function \onemax, the average number of iterations 
of algorithm~$\mathcal{A}$, starting from a solution $x\in Z_{s}$, until 
improvement of the objective function, under condition (\ref{Ca23}) 
satisfies the inequality 
\begin{equation} \label{C17} 
\mathbf{E}\ell_{n}(s)\leq C_{4}^{-1} n/s, \ s\in \mathbb{N}_{n},
\end{equation}
where $C_{4}$ is defined in expression (\ref{C16}). 
\end{lemma}

Let $\tau_{i}(n)$ be the number of iterations until the first 
hit of $x^*$ given that the process starts from an 
individual $x^{(0)}\in Z_{i}$. If the initial solution 
is chosen uniformly from the population, then 
\begin{equation} \label{C22} 
\mathbf{P}\big(x^{(0)}\in Z_{i}\big)=C_{n}^{i}2^{-n}, \quad 
\mathbf{E}\tau_{i}(n)\leq \sum_{s=1}^{i}\mathbf{E}\ell_{n}(s)
=\sum_{s=0}^{i}\mathbf{E}\ell_{n}(s).
\end{equation}

Let $\tau(n)$ be the number of iterations until the first hit of 
$x^*$ from a randomly chosen individual $x^{(0)}$. By the law of 
total probability, the representation holds
\begin{eqnarray} \nonumber 
\mathbf{E}\tau(n)&=&2^{-n}\sum_{i=0}^{n}C_{n}^{i}\mathbf{E}
\tau_{i}(n)\leq 2^{-n}\sum_{i=0}^{n}C_{n}^{i} 
\sum_{s=0}^{i}\mathbf{E} \ell_{n}(s)\\  \nonumber 
&=&2^{-n}\sum_{s=1}^{n}\mathbf{E}\ell_{n}(s)\sum_{i=s}^{n}
C_{n}^{i}=2^{-n}\sum_{s=1}^{n\epsilon}\mathbf{E}\ell_{n}(s)
\sum_{i=s}^{n} C_{n}^{i}\\  \label{C24} 
&+&2^{-n}\sum_{s=n\epsilon+1}^{n}\mathbf{E}\ell_{n}(s)
\sum_{i=s}^{n}C_{n}^{i}\leq \sum_{s=1}^{n\epsilon}
\mathbf{E}\ell_{n}(s)+C_{5}n,
\end{eqnarray}
where $\epsilon\in(0,1)$ is an arbitrary fixed number and the sum 
of combinations divided by $2^{n}$ is the sum of probabilities, which does not 
exceed~1, and $\mathbf{E}\ell_{n}(s)\leq C_{4}^{-1}
\epsilon^{-1} =:C_{5}$ for $n\epsilon+1\leq s\leq n$.
Here and below, in sums over subsets of natural numbers and 
functions of a natural argument with not necessarily integer 
limits or values, we consider that these quantities are equal to the 
nearest smaller integer.

\section{Average number of iterations until hitting the optimum}

\subsection{Regularly varying functions}

Let us consider several definitions and properties of regularly varying 
functions.
(See e.g. the monograph \citep{Se85} and \S 9 of Chapter VIII in~\citep{Fe2}.)

\begin{definition} \label{df1} 
A measurable function $g(v)>0$, defined for sufficiently large 
$v\in\mathbb{R}^{+}$ or $v\in\mathbb{N}$, is called 
\textbf{regularly varying at infinity} with exponent 
$\alpha\in\mathbb{R}$, if for any fixed $c\in \mathbb{R}^{+}$ the condition holds
\begin{equation} \label{u11} 
\lim_{v\to+\infty}\dfrac{g(cv)}{g(v)}\to c^{\alpha},
\end{equation}
where $cv$ should be replaced by $[cv]$,  the integer part of the number in the case 
$v\in\mathbb{N}$.

In the case $\alpha=0$, the function is called \textbf{slowly 
varying at infinity}.
\end{definition}

The property of regular variation is asymptotic and 
the function $g(v)>0$ does not necessarily have to be defined on any initial 
interval of the half-axis. Regularly varying functions for $v\in 
\mathbb{N}$ and $v\in\mathbb{N}$ will be called regularly varying 
sequences. (There is another definition
of regularly varying sequences for $c\in\mathbb{N}$ in
expression~(\ref{u11}), but to avoid unnecessary complications, we 
restrict ourselves to the simple case.)

Any regularly varying at infinity function
$g(v)$, $v\in\mathbb{R}^{+}$, will be asymptotically equivalent to
a regularly varying at infinity step function $g([v])$, which can be interpreted as a regularly varying at 
infinity sequence $g(z)$, $z\in \mathbb{N}$.
Obviously, the converse statement also holds.
\begin{lemma} \label{ldf0} 
For a regularly varying at infinity sequence $g(z)
$, $z\in \mathbb{N}$, there exists a regularly varying at 
infinity function $\hat{g}(v)$, $v\in\mathbb{R}^{+}$, 
such that $g([v])=\hat{g}([v])$ for sufficiently large $v\in 
\mathbb{R}^{+}$.
\end{lemma}
Further, we will not distinguish between notations for sequences and 
functions defined through each other, i.e., we identify the symbols 
$\hat{g}$ and~$g$.

\begin{definition}  \label{df2} 
A measurable function $g(v)>0$, defined for sufficiently small 
$v\in \mathbb{R}^{+},$ is called \textbf{regularly varying at 
zero} from the right with exponent $\alpha\in\mathbb{R}$, if for any 
fixed $c\in \mathbb{R}^{+}$ the condition holds
\begin{equation*} 
\lim_{v\to+0}\dfrac{g(cv)}{g(v)}\to c^{\alpha}
\end{equation*}

In the case $\alpha=0$, the function is called \textbf{slowly varying at zero} from the right.
\end{definition}

Definition~(\ref{df2}) can be formulated for any $a\in\mathbb{R}$, both in one-sided and two-sided 
variants, by replacing $g(v)>0$ with $g(v-a)>0$ 
and convergence in one sense or another of $v-a$ to zero.

Here are several examples of regularly varying functions. This 
class is a generalization of the family of power functions $g(v)=v^{\alpha}$, $v\in\mathbb{R}^{+}$, and sequences for
$v\in\mathbb{N}$.
Functions $v^{\alpha}$, $v\in\mathbb{R}^{+}$, for any $\alpha\in
\mathbb{R}$ are regularly varying at 0 and at infinity.
Functions $|\ln v|$, $|\ln|\ln v||$ and any powers of them are
slowly varying at 0 and at infinity.
If a slowly varying function at 0 or at infinity is 
denoted by $\ell(v)$, then for any $\alpha\in\mathbb{R}$ the function
$g(v)=v^{\alpha}\ell(v)$ will be regularly varying at zero or at infinity with exponent $\alpha$. Moreover, all regularly varying functions are representable only in this form.

\subsection{Generalization of conditions on population size and mutation settings}

The following two definitions specify our conditions on the distribution of the population 
size~$\lambda(n)$ (and hence on the distribution of the mutation parameter 
$p=\lambda(n)/n$).

\begin{definition} \label{df3_1}
A random sequence $\lambda(n)$
satisfies condition $\mathcal{A}_{2}^{c}$ if, for all $n\in\mathbb{N}$, the second moment satisfies the inequality
$\mathbf{E}\lambda^{2}(n)<C$ for some constant $C>0$.
\end{definition}

\begin{definition} \label{df3_2}
A random sequence $\lambda(n)$
satisfies conditions $\mathcal{A}_{2}^{\infty}$, if $\mathbf{E}\lambda^{2}(n)\geq \psi(n)= \mathcal{L}(u_{n})\to \infty$, where $\mathcal{L}(m)$ is a regularly varying sequence with exponent
$3-\beta\in (0.2)$ as $m\to\infty$, independent of $n$, $u_{n}\leq
n/2,$ and $u_{n}\to\infty$ is a regularly varying sequence as $n\to \infty$. Moreover, for some constant $C>0$ independent of $n$ and any $b\in\mathbb{N}_{u_{n}}\backslash
\mathbb{N}_{m_{0}}$ (for sufficiently large $n$$)$, the following inequality holds:
\begin{equation} \label{C18}
\mathbf{P}(b/2\leq\lambda(n)\leq b)\geq Cb^{-2}\mathcal{L}(b).
\end{equation}
\end{definition}

The aim of this work is to generalize Theorems~5 and~6 
from~\citep{ABD22} to wider classes of distributions of sequences 
of random variables $\lambda(n)$, $n\in \mathbb{N}$. Let us give the explicit 
form of distributions from \citep{ABD22}
\begin{equation} \label{nC1}
\mathbf{P}(\lambda(n)=k)=p_{n,k}=C_{\beta,u_{n}}k^{-\beta}, \ k\in \mathbb{N}_{u_{n}}, 
\end{equation}
where $u_{n}\to\infty$ as $n\to\infty$ and $C_{\beta,u_{n}}=
\sum_{k=1}^{u_{n}}k^{-\beta}$. We consider only the case 
$\beta>1$. Due to the convergence of the series $\sum_{k=1}^{\infty}k^{-\beta}$, 
the sequence $C_{\beta,u_{n}}$ is asymptotically constant. From 
the point of view of regularly varying functions, the probabilities $p_{n,k}$ 
from representation (\ref{nC1}) are described in terms of regularly 
varying functions with exponent $-\beta$, in which the slowly varying 
factor is an asymptotically constant function of $u_{n}$. The second moment $\mathbf{E}\lambda^{2}(n)= 
C_{\beta,u_{n}}\sum_{k=1}^{u_{n}}k^{2-\beta}$, as shown in 
 \citep{ABD22}, is of order $u_{n}^{3-\beta}$, i.e., it will be 
a regularly varying function with exponent $3-\beta$ for 
$\beta<3$. This motivates the choice of the parameter $3-\beta$ in 
definition $\mathcal{A}_{2}^{\infty}$.

Condition $\mathcal{A}_{2}^{c}$ means that the second moment 
of the random variable $\lambda(n)$ is uniformly bounded and there are no other 
constraints on $p_{n,k}$, which is realized for $\beta>3$ in condition
(\ref{nC1}).

Condition $\mathcal{A}_{2}^{\infty}$ means that the regularly 
varying at infinity function $\mathcal{L}(v)$ has the form 
$\mathcal{L}(v)=v^{3-\beta}\ell(v)$, $v\in\mathbb{R}^{+}$, where 
$\ell(v)$ varies slowly at infinity. Condition (\ref{C18}) 
can be written as
$$
\mathbf{P}(b/2\leq\lambda(n)\leq b)\geq Cb^{1-\beta}\ell(b).
$$
In terms of assumptions (\ref{nC1}), these conditions hold for $1<
\beta<3$ and $\ell(v)$ asymptotically constant, but the explicit form 
of $p_{n,k}$, the probabilities of specific values for $\lambda(n)$, is not important, and only bounds for their sums over long 
intervals are used, which will be regularly varying at infinity
functions with exponent $1-\beta$.
Obviously, $\mathcal{L}(u_{n})$, as a composition of regularly 
varying sequences, will be a regularly varying 
sequence as $n\to\infty$ with a naturally computed exponent.

Condition $\mathcal{A}_{2}^{\infty}$ could also include 
some cases $\beta=-1,-3$, which yield explicit 
bounds for $\mathbf{E}\tau(n)$, but these proofs are 
cumbersome and require working with subtle structural theorems for 
slowly varying functions.
Let us give a generalization of Theorem 5 from \cite{ABD22}.
\begin{theorem} \label{t1}
The average number of iterations $\tau$ in algorithm $(1 + (\lambda, \lambda))$~GA with the fitness function \onemax before reaching the optimum is  
\begin{eqnarray} \label{ot1}
\mathbf{E}\tau(n)&=&O(n\ln n), \mbox{ subject to } \mathcal{A}_{2}^{c};\\ \label{ot2}
\mathbf{E}\tau(n)&=&O(n)+O\left (\dfrac{n\ln n}{\psi(n)}\right),
\mbox{ subject to } \mathcal{A}_{2}^{\infty},
\end{eqnarray}
where $\psi(n)$ is from definition~\ref{df3_2}.
\end{theorem}

Bound (\ref{ot2}) for the number of iterations until the first 
hit of $x^*$ when passing to arbitrary regularly varying at infinity 
sequences $u_{n}$ and $\mathcal{L}(n)=n^{1-\beta}\ell(n)$
differs qualitatively from that obtained in Theorem~5 from~\citep{ABD22},
where the second term is absent, which disappears when $\psi^{-1}
(n)\ln n=\mathcal{L}^{-1}(u_{n})\ln n=O(1)$. In the case of 
conditions (\ref{nC1}), the last bound becomes $u_{n}^{\beta
-3}\ln n=O(1)$, which implies condition $u_{n}\geq \ln^{1/(3-\beta)} n$ of Theorem~5~\citep{ABD22}.

\section{Upper bounds for optimization time and computational cost}

Let $T^{op}$ be the number of computational operations until 
the first hit of $x^*$ in the RAM model with 
random access memory~\citep{AHU}, where standard 
arithmetic operations have constant duration. The distributions of the random variables $T$ and $T^{op}$ depend on
the dimension of the bit strings of individuals $n$. Therefore, for them, we further 
use the notations $T(n)$ and $T^{op}(n)$, respectively.
Denote by $T_{s}(n)$ and $T_{s}^{op}(n)$ the number of fitness function 
evaluations and computational operations until the first 
hit of $x^*$, respectively, given that the process 
starts from an individual $x^{(0)}\in Z_{s}$.

Denote the number of fitness function evaluations and the number 
of computational operations in one iteration, starting from an individual 
$x\in Z_{s}$, ending with the preservation, possibly of a different 
individual $x$ from $Z_{s}$ for the next operation, for fixed 
$\lambda_{i,s}(n)$, by $\mu_{\lambda(n)}(i,s)$ and $
\nu_{\lambda(n)}(i,s,n)$. Here the random variables 
$\lambda_{i,s}(n)$ are independent in $i$ and $s$ and have the same 
distribution as $\lambda(n)$, and the index $i$ denotes the 
ordinal number of a start from an individual $x$ (possibly different from 
the original one, but with the same fitness function value) with 
an unsuccessful outcome -- without transition to a higher level. 
The averages over $\lambda_{i,s}(n)$ for $\mu_{\lambda(n)}(i,s)$ and 
$\nu_{\lambda(n)}(i,s, n)$ are denoted 
by $\mu_{n}(i,s)=\mathbf{E}_{\lambda(n)}\mu_{\lambda(n)}(i,s)$ 
and $\nu_{n}(i,s)=\mathbf{E}_{\lambda(n)}\nu_{\lambda(n)}(i,s,n)$, 
respectively. Due to the identical distribution of the r.v.s $\lambda_{i,s}(n)$ in $i$,
the latter averages coincide for different $i$
with fixed $s$ and $n$.

For $\mu_{\lambda(n)}(i,s)$ in the first step of the loop (mutation 
of the individual $x$), $\lambda(n)$ evaluations of the 
fitness function $\mathbf{E}\lambda(n)\leq\mu_{n}(i,s)$ are performed, and 
in the second step, corresponding to crossover, there are no more 
than $\lambda(n)$ such evaluations. Therefore, the constraints hold
\begin{equation}   \label{n043}
\mu_{n}(i,s)\leq 2\mathbf{E}\lambda(n).
\end{equation}

Estimates for $\nu_{\lambda(n)}(i,s)$ are more complicated. This value
depends on the implementation of the computational algorithm. Suppose that 
there exists a function $\phi(\lambda(n),n)$ such that 
$$
\nu_{\lambda(n)}(i,s,n)\leq\phi(\lambda_{i,s}(n),n).
$$

Denoting $\phi_{n}=\mathbf{E}_{\lambda(n)}
\phi(\lambda_{i,s}(n),n)$, we obtain the inequalities
\begin{equation}   \label{n43}
\nu_{n}(i,s)\leq \phi_{n}.
\end{equation}

Let $\mu_{\lambda(n)}(s)$ and $\nu_{\lambda(n)}(s,n)$ denote the 
random numbers of fitness function evaluations and 
computational operations in one iteration starting from any 
individual $x$ in $Z_{s}$, ending with a transition to another 
individual with a higher fitness function value. Their averages over $\lambda(n)$ 
are denoted by $\mu_{n}(s)=\mathbf{E}_{\lambda(n)}
\mu_{\lambda(n)}(s)$ and $\nu_{n}(s)=\mathbf{E}_{\lambda(n)}
\nu_{\lambda(n)}(s, n)$, respectively. For these 
averages, inequalities analogous to (\ref{n043}) and (\ref{n43}) hold
\begin{equation}   \label{n44}
\mu_{n}(s)\leq 2\mathbf{E}\lambda(n),
\ \  \nu_{n}(s)\leq \phi_{n}.
\end{equation}

Let $\zeta_{n}(s)$ and $\eta_{n}(s)$ be the number of objective function 
evaluations and the number of computational operations during the time 
that successive individuals $x$ stay in $A_{s}$, ending with a transition 
to an individual with a higher fitness function value. More specifically, the representations hold
\begin{equation}   \label{C43}
\zeta_{n}(s)=\sum_{i=1}^{\ell_{n}(s)-1}\mu_{n}(i,s)+\mu_{n}(s),\ \  
\eta_{n}(s)=\sum_{i=1}^{\ell_{n}(s)-1}\nu_{n}(i,s)+\nu_{n}(s), 
\end{equation}
where the sum is zero if the upper index is less than the lower one.
%


Relations (\ref{n043}), (\ref{n43}), (\ref{n44}) and
(\ref{C43}) and Kolmogorov-Prokhorov theorem \citep[Ch. 4, \S 4]{Bo76} for non-negative random variables
allow us to write the inequalities
\begin{eqnarray} \label{b1}
\mathbf{E}\ell_{n}(s)\mathbf{E}
\lambda(n)\leq 2\mathbf{E}\ell_{n}
(s)\mathbf{E}\lambda(n), \ \  
\mathbf{E}\eta_{n}(s)\leq \mathbf{E}
\ell_{n}(s)\phi_{n}.
\end{eqnarray}

The next theorem is a generalization of Theorem 6 from \citep{ABD22}.

\begin{theorem} \label{t2} 
The average number of fitness function evaluations 
$T(n)$ and the number of operations $T^{op}(n)$ in algorithm 
$\mathcal{A}$ with fitness function \onemax are bounded 
from above by the quantities
\begin{eqnarray} \nonumber
\mathbf{E} T(n)&=&O(n\ln n), \mbox{ under condition } 
\mathcal{A}_{2}^{c};\\   \label{ott3}
\mathbf{E}T(n)&=&O\left (n+\dfrac{n\ln n}{\psi(n)}
\right)\mathbf{E}\lambda(n), \mbox{ under condition } \, 
\mathcal{A}_{2}^{\infty};\\ \nonumber
\mathbf{E}T^{op}(n)&=&O(n\ln n)\phi_{n}, \mbox{ under condition } 
\mathcal{A}_{2}^{c};\\  \nonumber
\mathbf{E}T^{op}(n)&=&O\left (n+\dfrac{n\ln n}{\psi(n)}
\right)\phi_{n}, \mbox{ under condition } \, \mathcal{A}_{2}^{\infty}.
\end{eqnarray}
\end{theorem}


The statement of Theorem 6 \citep{ABD22} corresponds to the special case of~bounds~(\ref{ott3}) under the conditions 
\begin{equation*} 
\mathbf{E}\lambda(n)\leq C<\infty,\ \forall n\in\mathbb{N},
\end{equation*} 
(or $\beta>2$), $\psi(n)=u_{n}^{3-\beta}$ and $u_{n}\geq 
\ln^{1/(3-\beta)} n$, which is a special case of the condition $u_{n}^{\beta-3} \ln n=O(1)$.

{
\begin{corollary} \label{cor:2}
If we choose the set $\mathbb{N}^{(2)}=\{2^{\ell}, \ell=0,1,2,\cdots\}$ as the support of the distribution of $\lambda(n)$ and set $p_{n,k}:=C^{(*)}(\beta,u_{n})k^{1-\beta}$
for $k\in \mathbb{N}_{u_{n}}^{(2)}$ and $u_{n}\in\mathbb{N}^{(2)}$, then $\mathbf{E} T(n)=O(n)$ in the case of
the fitness \onemax.
\end{corollary}

In the next section, an experimental comparison of CPU time until the first reaching of the optimum of the
\onemax function is carried out when generating the r.v. $\lambda(n)$ as indicated in Corollary~\ref{cor:2} or according to
the power law with support $\{0,1,2,\dots,n\}$ from~\citep{ABD22}.

\section{Computational experiment}

In the computational experiment, algorithm~$\mathcal{A}$ was considered, where the r.v. $\lambda(n)$ was generated 
as indicated in Corollary~\ref{cor:2} (hereinafter, algorithm~A) and the algorithm $(1+(\lambda,\lambda))$~GA, where $\lambda(n)$ is chosen
according to the power law with support $\{0,1,2,\dots,n\}$ from~\citep{ABD22} (hereinafter, algorithm~B). 
The optimization criterion was the function \onemax, and the algorithm had an adjustable parameter $\beta=2.75$ and an upper bound 
on the number of offspring $u_n=n$.

The software code in Scala proposed by the authors of~\citep{ABD22} was used as a basis. Changes concerned only the procedure
for generating~$\lambda(n)$: in the case of algorithm~A, it can be implemented with complexity~$O(\log \log(u_n))$,
while in the original implementation from~\citep{ABD22}, choosing $\lambda$ requires $O(\log(u_n))$ operations.
In order to compare the computational costs in CPU time until the first reaching the optimum, experiments were carried out
for $n=2^{15}, 2^{16}, 2^{17}$, $2^{18}$ and $2^{19}$. Computations were performed on a server 
with an AMD EPYC 7502 processor using 5 independent parallel threads. For each instance,
both algorithms were run $10^5$ times. The arithmetic mean~$\hat{T}_A, \hat{T}_B$ and the estimates of standard
deviation~$\hat{\sigma}_A, \hat{\sigma}_B$ for the measured computation time~(milliseconds) are given in Table~1. 
As can be seen from the table, 
algorithm~A has an advantage in average computation
time and smaller standard deviation.

\begin{table}[h!]
\centering
\begin{tabular}{ |p{1cm}||p{1cm}|p{1.2cm}|p{1.2cm}|p{1.3cm}|p{1.4cm}|}
 \hline
 $n$ & $2^{15}$ & $2^{16}$ & $2^{17}$ & $2^{18}$ & $2^{19}$  \\
 \hline
 $\hat{T}_A$  & 93.91  & 201.65  & 476.82  & 1141.03  & 2790.51  \\
 \hline
 $\hat{\sigma}_A$ & 225.58 & 593.88  & 2005.57 & 6629.19  & 22902.68 \\
 \hline
 $\hat{T}_B$  & 117.11 & 266.83  & 608.05  & 1492.12  & 3783.67  \\
 \hline
 $\hat{\sigma}_B$ & 420.09 & 1495.47 & 5007.79 & 18652.82 & 64356.55 \\
 \hline
\end{tabular}
\caption{Average CPU time (ms) until first hitting the optimum and the estimate
of its standard deviation for algorithms~A and~B depending on the problem size.}
\label{table:1}
\end{table}


\section{Conclusion}

In this paper, we investigated the question of the average time to obtain the optimal solution of a simple model problem of maximizing~\onemax 
using a known version of the genetic algorithm $(1+(\lambda, \lambda))$~GA. 
The main result of this work is that the upper bound obtained in the article by Antipov, Buzdalov and 
Doerr~\citep{ABD22} for $(1+(\lambda, \lambda))$~GA with the fast mutation operator remains valid in a more general case, 
when the distributions for the size of the intermediate 
population~$\lambda$ and the mutation parameter~$p$ are chosen from a wider class of distributions.
The conducted computational experiment showed promising results, suggesting
that the proposed method of random selection of the population size~$\lambda(n)$ may be useful in practice.

The function~\onemax considered here has only one local optimum, 
which is also global. As follows from the theoretical results obtained
in~\citep{DLMN17} for $(1+(\lambda, \lambda))$~GA on the model function \jump with many local optima, the use
of fast mutation on this function removes the problem of exact tuning of the population size and mutation parameter.
In this regard, in further research it makes sense to consider the possibilities of relaxing the requirements for the distribution
of these parameters in the fast mutation operator when optimizing the \jump function and other multimodal functions,
in particular, when solving NP-hard pseudo-Boolean optimization problems.

}

\subsubsection*{Acknowledgments}
This research was
funded by the Mathematical Center in Akademgorodok under the agreement \textnumero~075-15-2025-349 with the Ministry of Science and Higher Education of the Russian Federation.

\bibliography{references}
\bibliographystyle{mathai2026_conference}

\section*{Appendix}
This appendix contains the proofs omitted from the main text of  the paper. 

{\bf Lemma~\ref{l2}.} {\em
In the case of the fitness function \onemax, the average number of iterations 
of algorithm~$\mathcal{A}$, starting from a solution $x\in Z_{s}$, until 
improvement of the objective function, under condition (\ref{Ca23}) 
satisfies the inequality 
\begin{equation} \label{C17a} 
\mathbf{E}\ell_{n}(s)\leq C_{4}^{-1} n/s, \ s\in \mathbb{N}_{n},
\end{equation}
where $C_{4}$ is defined in expression (\ref{C16}). 
}

{\bf Proof.} Estimate (\ref{C17a}) follows from relations 
(\ref{C15}) and~(\ref{C16}).
Q.E.D.
$\mbox{ }$
\linebreak

{\bf Theorem~\ref{t1}.} {\em
The average number of iterations $\tau$ in algorithm $(1 + (\lambda, \lambda))$~GA with the fitness function \onemax before reaching the optimum is  
\begin{eqnarray} \label{ot1a}
\mathbf{E}\tau(n)&=&O(n\ln n), \mbox{ subject to } \mathcal{A}_{2}^{c};\\ \label{ot2a}
\mathbf{E}\tau(n)&=&O(n)+O\left (\dfrac{n\ln n}{\psi(n)}\right),
\mbox{ subject to } \mathcal{A}_{2}^{\infty},
\end{eqnarray}
where $\psi(n)$ is from definition~\ref{df3_2}.
}

{\bf Proof.}
Note that under condition $\mathcal{A}_{2}^{c}$, the first moment $\mathbf{E}\lambda(n)$ is also uniformly bounded in $n$. Consequently, by Lemma \ref{l2}, under conditions $\mathcal{A}_{2}^{c}$, the inequality holds
\begin{eqnarray} \label{C280} 
\mathbf{E}\ell_{n}(s)\leq C_{4}^{-1} \dfrac{n}{s}. 
\end{eqnarray}

Note that the sum $\sum_{s=1}^{u}s^{-1}$ of a regularly varying  step function as $u\to\infty$ will be asymptotically  equivalent to the integral $\int_{1}^{u}x^{-1}dx$, i.e., as $u\to \infty$
\begin{equation} \label{01C}
\sum_{s=1}^{u}s^{-1}\sim \int_{1}^{u}x^{-1}dx= \ln u. 
\end{equation}
Estimate (\ref{ot1a}) follows from relations (\ref{C24}),  (\ref{C280}) and (\ref{01C}).

Now we prove inequality~(\ref{ot2a}).

Consider first the case $u^{2}_{n}s/n<1$. Under condition $\mathcal{A}_{2}^{\infty}$,  bound (\ref{C13}) leads to the inequality
\begin{equation} \label{C20} 
p_{n}(s)=\mathbf{E}_{\lambda(n)}p_{\lambda(n)}(s)\geq C^{*} \mathcal{L}(u_{n})s/n=C^{*}\psi(n)s/n. 
\end{equation}

Therefore, for $u^{2}_{n}s/n<1$ and under condition $\mathcal{A}_{2}^{\infty}$,  from relations (\ref{C15}) and (\ref{C20}) we obtain 
\begin{eqnarray} \label{C28} 
\mathbf{E}\ell_{n}(s)\leq C^{*} \dfrac{n}{\psi(n)s}, 
\end{eqnarray}
where formally instead of $C^{*}$ should be $1/C^{*}$ with the constant  from bound (\ref{C20}), but in accordance with our assumptions we use the same notation $C^{*}$ for it.

This bound for $s>n\epsilon,$ where $\epsilon\in(0,1)$ is an arbitrary fixed number, and $u^{2}_{n}\to\infty,$ is not  applicable, but if $s>n\epsilon$, then the sum containing such terms $\mathbf{E}\ell_{n}(s)$ has already been boundd above by the quantity  $Cn$ in relations (\ref{C24}) and cannot be improved in order.

Now consider the case $u^{2}_{n}s/n \ge 1$. In this case, under conditions $\mathcal{A}_{2}^{\infty}$, from inequality (\ref{C18}) for $m_{0}<b\leq u_{n}$ follows the bound 
\begin{equation}  \label{C19} 
\mathbf{E}_{\lambda(n)}\lambda^{2}(n)\geq \mathbf{E}_{\lambda(n)}
\big\{\lambda^{2}(n);\lambda(n)\in[b/2,b]\big\}\geq C\mathcal{L}(b).
\end{equation}

In the case under consideration, for sufficiently large~$n$, the  inequalities $m_{0}<\sqrt{n/s}\leq u_{n}$ hold. Therefore, using  inequality~(\ref{C19}) with $b=\sqrt{n/s}$, $s\in  \mathbb{N}_{\epsilon n}$, we bound the average $p_{\lambda(n)}(s)$ on the set $\lambda^{2}(n)<n/s$    
\begin{equation*} 
\mathbf{E}_{\lambda(n)}p_{\lambda(n)}(s)\geq C \mathcal{L}(b)s/n.
\end{equation*}

From here and from relation (\ref{C15}) we obtain 
\begin{eqnarray} \label{C27} 
\mathbf{E}\ell_{n}(s)\leq C^{*} \dfrac{n}{\mathcal{L}\big(\sqrt{n/s}\big)s}.
\end{eqnarray}

Inequalities (\ref{C28}) and (\ref{C27}), obtained in the two 
cases considered above, imply the bound
\begin{eqnarray}   \label{line1}
\sum_{s=1}^{n\epsilon}\mathbf{E}\ell_{n}(s)&\leq& 
\frac{C^{*}n}{\psi(n)} \sum_{s=1}^{n/u^{2}_{n}-1}s^{-1} + C^{*}
\sum_{s=n/u^{2}_{n}}^{n\epsilon}\dfrac{n}{\mathcal{L}\big(\sqrt{n/ s}\big)s}\\ \label{C29} &\leq& C^{*}\dfrac{n}{\psi(n)}\ln (n/u^{2}_{n})+ C^{*}_{1 } n.
\end{eqnarray}

We explain inequality (\ref{C29}). The first sum from the  right-hand side of expression~(\ref{line1}) is bounded here using  relation~(\ref{01C}). The terms of the second sum from~(\ref{line1}) can be represented via the function 
$\mathcal{L}_{0}(x):=x\mathcal{L}^{-1}(\sqrt{x}),$ where $x=n/s$, 
which will be regularly varying at infinity with exponent 
$\beta_{0}:=1-0.5(3-\beta)=0.5(\beta-1)$. Under the change of variables
$y=x^{-1}$, the function $\mathcal{L}_{0}(y^{-1})$ will be regularly 
varying at zero with exponent $-1<-\beta_{0}<0$.
From the function $\mathcal{L}_{0}(y^{-1})$, we determine an equivalent
regularly varying at zero function $\tilde{\mathcal{L}}_{0}
(y^{-1})$ with values $\mathcal{L}_{0}(s/n)$ on the set 
$y^{-1}\in[s/n,(s+1)/n)$. The sum of a regularly varying 
sequence $\mathcal{L}_{0}(s/n)$ coincides with the integral 
of $\tilde{\mathcal{L}}_{0}(y^{-1})$, which is constant on 
half-intervals of length $n^{-1}$ and will be regularly varying at zero 
with exponent $-1<-\beta_{0}<0$. After the change of variables 
$u=yn^{-1}$, the last integral becomes
$$
n\int_{u^{-2}_{n}}^{\epsilon} \tilde{\mathcal{L}}_{0}(u)du\leq
n\int_{0}^{\epsilon} \tilde{\mathcal{L}}_{0}(u)du,
$$
which converges at zero for $-1<-\beta_{0}<0$ by the lemma from 
\citep[Ch. VIII, \S 9]{Fe2}. The uniform boundedness of 
the second sum is proved, which completes the proof of 
bound~(\ref{C29}).

If $\psi(n)$ varies regularly with a positive exponent, then 
the first term in (\ref{C29}) will be 
of order $o(n)$, and the integral will converge. As a result, under 
conditions $\mathcal{A}_{2}^{\infty}$ for functions 
$\psi(n)$ that are not slowly varying, we have
\begin{equation} \label{07}
\sum_{s=1}^{n\epsilon}\mathbf{E}\ell_{n}(s)\leq C^{*} n.
\end{equation}

For slowly varying unboundedly growing 
functions $u_{n}$, by property $2^{\circ}$ from~\citep[Ch.~1, Sec.~1.5]{Se85}, the relation $\ln u_{n}=o(\ln n)$ holds, which 
together with bounds (\ref{C24}) and (\ref{C29}) proves 
relation (\ref{ot2a}). Q.E.D.

$\mbox{ }$\linebreak

{ \bf Theorem~\ref{t2}.}{\em 
The average number of fitness function evaluations 
$T(n)$ and the number of operations $T^{op}(n)$ in algorithm 
$\mathcal{A}$ with fitness function \onemax are bounded 
from above by the quantities
\begin{eqnarray} \nonumber
\mathbf{E} T(n)&=&O(n\ln n), \mbox{ under condition } 
\mathcal{A}_{2}^{c};\\   \label{ott3a}
\mathbf{E}T(n)&=&O\left (n+\dfrac{n\ln n}{\psi(n)}
\right)\mathbf{E}\lambda(n), \mbox{ under condition } \, 
\mathcal{A}_{2}^{\infty};\\ \nonumber
\mathbf{E}T^{op}(n)&=&O(n\ln n)\phi_{n}, \mbox{ under condition } 
\mathcal{A}_{2}^{c};\\  \nonumber
\mathbf{E}T^{op}(n)&=&O\left (n+\dfrac{n\ln n}{\psi(n)}
\right)\phi_{n}, \mbox{ under condition } \, \mathcal{A}_{2}^{\infty}.
\end{eqnarray}
}

{\bf Proof.}  
Applying the law of total probability by analogy with (\ref{C24}) and using inequalities (\ref{b1}), we have
\begin{eqnarray}   \label{C51}
&\mathbf{E}T(n)=2^{-n}\sum\limits_{s=0}^{n-1}
C_{n}^{s}\mathbf{E}T_{s}(n)\leq C^{*} \mathbf{E}
\tau(n)\mathbf{E}\lambda(n),\\   \label{C52}
&\mathbf{E}T^{op}(n)=2^{-n}\sum\limits_{s=0}^{n-1}
C_{n}^{s}\mathbf{E}T^{op}_{s}(n)\leq C^{*} \mathbf{E}
\tau(n)\phi_{n}.
\end{eqnarray}

Theorem \ref{t1} and bounds (\ref{C51}) and (\ref{C52}) imply the statement of Theorem \ref{t2}. Q.E.D.

\end{document}

%% file: mathai2026_math_commands.tex
\usepackage{amsmath,amsfonts,bm}

\def\eqref#1{equation~\ref{#1}}

\def\1{\bm{1}}

\DeclareMathAlphabet{\mathsfit}{\encodingdefault}{\sfdefault}{m}{sl}
\SetMathAlphabet{\mathsfit}{bold}{\encodingdefault}{\sfdefault}{bx}{n}



%% file: preamble.tex
\usepackage{amsfonts}
\usepackage{amsmath}
\usepackage{url}

\usepackage{xspace}                           
\usepackage{booktabs}                         
\usepackage{marginnote}

\usepackage{color} 

\allowdisplaybreaks

\DeclareMathOperator*{\Bin}{Bin}

\usepackage{mathrsfs}

\usepackage{bbm}

\usepackage{algorithm,algorithmic}

\floatname{algorithm}{Алгоритм}

\newcommand{\onemax}{\text{\sc OneMax}\xspace}

\newcommand{\jump}{\text{\sc Jump}\xspace}

\usepackage{enumitem}

